\documentclass[authoryear,10pt,letterpaper]{elsarticle}
\usepackage[top=0.75in, bottom=0.75in, left=0.75in, right=0.75in]{geometry}

\usepackage{palatino}
\usepackage{amsmath,amsthm}
\usepackage{amssymb}
\usepackage{bbm}
\usepackage{amsfonts}
\usepackage{graphicx}
\usepackage{subfigure}
\usepackage{algorithm}
\usepackage{algorithmic}
\usepackage{tabularx,booktabs}

\usepackage{tablefootnote}

\usepackage{xcolor}
\definecolor{darkblue}{rgb}{0.0,0.5,0.5}
\definecolor{blue}{rgb}{0.0,0.5,0.68}
\usepackage[colorlinks]{hyperref}
\hypersetup{colorlinks,breaklinks,linkcolor=darkblue,urlcolor=darkblue,anchorcolor=darkblue,citecolor=darkblue}
\usepackage{booktabs,caption}
\usepackage{multirow}
\usepackage{lscape}
\usepackage{epstopdf}
\usepackage{lineno}
\usepackage{subfigure}

\usepackage{microtype}
\usepackage{lineno}
\newdefinition{definition}{Definition}
\newtheorem{theorem}{Theorem}

\journal{ }
\makeatletter
\def\ps@pprintTitle{%
   \let\@oddhead\@empty
   \let\@evenhead\@empty
   \let\@oddfoot\@empty
   \let\@evenfoot\@oddfoot
}
\makeatother

\begin{document}

\begin{frontmatter}



\title{{\fontfamily{lmss}\selectfont A Nonconvex Low-Rank Tensor Completion Model for Spatiotemporal Traffic Data Imputation}}

\author[label1]{Xinyu Chen}
\ead{chenxy346@gmail.com}

\author[label2]{Jinming Yang}
\ead{yangjm67@gmail.com}

\author[label1]{Lijun Sun\corref{cor1}}
\ead{lijun.sun@mcgill.ca}

\address[label1]{Department of Civil Engineering, McGill University, Montreal, QC H3A 0C3, Canada}
\address[label2]{School of Naval Architecture, Ocean and Civil Engineering, Shanghai Jiao Tong University, Shanghai 200240, China}

\cortext[cor1]{Corresponding author. Address: 492-817 Sherbrooke Street West, Macdonald Engineering Building, Montreal, Quebec H3A 0C3, Canada}

\begin{abstract}
Sparsity and missing data problems are very common in spatiotemporal traffic data collected from various sensing systems. Making accurate imputation is critical to many applications in intelligent transportation systems. In this paper, we formulate the missing data imputation problem in spatiotemporal traffic data in a low-rank tensor completion (LRTC) framework and define a novel truncated nuclear norm (TNN) on traffic tensors of location$\times$day$\times$time of day. In particular, we introduce an universal rate parameter to control the degree of truncation on all tensor modes in the proposed LRTC-TNN model, and this allows us to better characterize the hidden patterns in spatiotemporal traffic data. Based on the framework of the Alternating Direction Method of Multipliers (ADMM), we present an efficient algorithm to obtain the optimal solution for each variable. We conduct numerical experiments on four spatiotemporal traffic data sets, and our results show that the proposed LRTC-TNN model outperforms many state-of-the-art imputation models with missing rates/patterns. Moreover, the proposed model also outperforms other baseline models in extreme missing scenarios.

%we consider the problem of recovering missing values within spatiotemporal traffic data in the demonstrated realistic missing patterns.

%To solve this problem, we define and with this, we propose a nonconvex TNN minimization based low-rank tensor completion model.
%Recent studies have shown that low-rank tensor completion (LRTC) in recovering missing values in

%Spatiotemporal traffic data imputation is an essential task for improving the quality of imperfect and sometimes highly incomplete raw data, however, recovering missing values with various missing rates/patterns is still very challenging.

\end{abstract}


\end{frontmatter}

\section{Introduction}

%\subsection{Motivation}

Spatiotemporal traffic data, which registers time-stamped traffic state (e.g., flow and speed) observations from different locations (e.g., a network of sensors), serves as a critical input to a wide range of applications in intelligent transportation systems (ITS), such as travel time estimation, trip planning, traffic forecasting, to name but a few. With the recent development in information and communication technology (ICT), the scale and dimension of spatiotemporal traffic data also become larger. In the meanwhile, however, these data sets also suffer from the missing data problem, which undermines their utility and effectiveness in real-world applications. In practice, the missing data problem may arise from sensor malfunctioning, communication failure, and maintenance. Another critical reason that causes the missing data problem is due to the insufficient sensor coverage in both the spatial and the temporal dimensions. For example, floating cars in an urban transportation network can provide a good sample of traffic information in real-time. However, the data itself is in nature sparse and far from enough to support city-wide and fine-resolution traffic speed and travel time monitoring. Therefore, performing imputation on spatiotemporal traffic data has become a critical step before further applications. Recent research has shown an increasing interest in developing efficient and effective missing data imputation methods for large-scale and high-dimensional traffic data \citep[see e.g., ][]{zhu2012compressive,li2013efficient,asif2016matrix,tan2016short,chen2019missing,chen2019abayesian,yu2020urban}.

%Data incompleteness introduces noise into the learning process of intelligent transportation systems, which may further lead to false feature extraction, inaccurate prediction or poor decision.

\begin{figure}[!ht]
\centering
  \includegraphics[width=\textwidth]{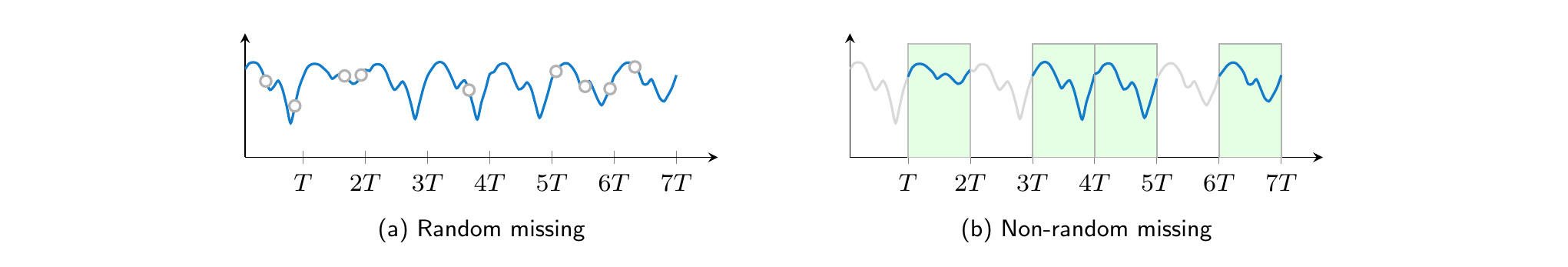}
\caption{Missing pattern examples of spatiotemporal traffic data (e.g., traffic speed). For a given spatial location, collected traffic data are indeed time series. In both two examples, $T$ is the total number of time intervals per day. (a) Data are missing at random. Small circles represent the missing values. (b) Data are missing not at random. This is a more common and realistic missing pattern for spatiotemporal traffic data in which data are missing continuously during few time periods. Blue curves within green plates represent the observed values, while gray curves represent the missing values.}
\label{missing}
\end{figure}

Incomplete spatiotemporal traffic data are multivariate/multidimensional time series with various missing patterns and missing ratios. Essentially, the missing patterns can be summarized into two types: random missing and non-random missing. Figure~\ref{missing} presents two intuitive examples of missing patterns (i.e., random missing and non-random missing) in a spatiotemporal setting. The fluctuation of the power grid and the loss of the transmission packets can lead to data random missing scenarios in which individual data points are randomly lost. While the sensor malfunctioning and regular maintenance may cause non-random data missing scenarios in which we have correlation corruption for a certain period of time. It has been demonstrated that non-random missing patterns are more difficult to deal with than random missing due to the correlated corruption. It is very challenging to impute the incomplete traffic data under non-random or mixture missing patterns that data may be lost for several hours or even several days.

The fundamental challenge of missing data imputation is to effectively borrow information from those ``observed'' entries by characterizing the higher-order correlations and dependencies in the data. To achieve this, recent research has shown an increasing interest in taking advantage of the low-rank property of spatiotemporal data, which allows us to apply techniques such as compressive sensing, principal component analysis (PCA), and matrix/tensor factorization. For example, \citet{zhu2012compressive} imposed the low-rank assumption on a traffic state matrix from a list of road segments over certain time and transforms missing data imputation to a matrix completion problem; based on probabilistic PCA, \citet{qu2009ppca} and \citet{li2013efficient} developed an improved model to better characterize the nonlinear spatiotemporal dependencies; based on probe vehicle data, \citet{yu2020urban} developed a city-wide traffic data estimation framework by performing matrix completion using Schatten $p$-norm.

Traffic time series data is unique in the sense that there exists strong daily patterns and similarity \citep{li2015trend}. To better capture the correlations and utilize the global structure across different days, \citet{asif2016matrix}, \citet{tan2016short} and \citet{chen2018spatial} model the temporal dimension as the combination of day and time of day, and then apply tensor factorization techniques to impute missing values. In general, these models fold a long time series vector from a sensor to a matrix by stacking the vectors from different days; thus, these works transform the original time series matrix to a third-order tensor by introducing ``day'' as an additional dimension. Given that most traffic data sets are inherently low-rank, these low rank-based models have demonstrated superior performance over other methods.

In the meanwhile, substantial research progress has been made in developing new learning methods for low-rank tensor completion (LRTC). Given that the rank operator is neither convex nor continuous, the nuclear norm (NN) is introduced as a convex surrogate for the rank function to be minimized. LRTC with a NN function has attracted considerable attention recently, and many variants have been developed in the past few years \citep[see e.g.,][]{liu2013tensor}. However, despite the remarkable performance of NN in matrix/tensor completion, it is also suggested that the convex NN relaxation may not be sufficient to approximate the rank, as all the singular values are simultaneously minimized \citep{hu2013fast}. It has been proved that solving the convex NN regularized problem leads to a near-optimal low-rank solution if we impose certain incoherence assumptions on the singular values of the matrix \citep{candes2010thepower}; however, such conditions are rarely met in practice. As a result, recent research has tried to identify alternative approximation to the rank function, and some studies have shown the advantage of using nonconvex surrogate functions to approximate the rank for matrices
\citep[see e.g.,][]{hu2013fast,lu2014generalized,lu2015generalized,yao2019largescale,fan2019factor}. For example, \citet{lu2014generalized} and \citet{yao2019largescale} evaluated several nonconvex singular value functions for approximating the rank function on low-rank matrix learning, and it is shown that nonconvex surrogate regularization can work better than NN. Recent studies have also extended matrix TNN to a tensor setting \citep[e.g.,][]{huang2014truncated,han2017sparse,xue2018low}, showing that nonconvex TNN minimization is more capable than convex NN minimization on tensor completion tasks. However, these studies mainly focus on an image setting (e.g., a tensor of 256$\times$256$\times$3 ($n_1\times n_2 \times n_3$)) in which $n_3\ll n_1, n_2$, and it remains a critical challenge to properly define the degree of truncation on each tensor mode in the case of spatiotemporal traffic data without prior knowledge. To solve this problem, we propose a new TNN model in which we automatically determine the nuclear norm truncation along each mode of a given data tensor by introducing a global rate parameter. Our contribution is twofold:
\begin{itemize}
    \item We develop an imputation model that benefits from our domain knowledge. Following the nonconvex TNN minimization on matrix learning problems, we propose a TNN minimization based LRTC model with elegantly defined TNN on tensors and apply an efficient optimization algorithm---Alternating Direction Method of Multipliers (ADMM)---to solve this model. The truncation along all tensor modes can be automatically determined by a rate parameter in this model, and this technical operation is more flexible than what given in previous LRTC-TNN models.

    \item We show the capability and advantage of the proposed model with extensive data sets and imputation experiments. Relying on the four selected spatiotemporal traffic data sets, we demonstrate the superiority (e.g., accuracy, efficiency) of the proposed model with various missing rates/patterns over the state-of-the-art imputation models. In a very unusual case that data have a large number of missing values (e.g., missing rate is 70\%), we also show that the proposed model provides accurate imputation results and outperforms the competing models.

    % In addition, we discuss the imputation performance of both proposed model and competing model for the data with a large number of missing values, and we find that the proposed model achieves the most accurate result.

    % \item Developing a model with strong background. Following the framework of tensor completion for spatiotemporal traffic data imputation summarized by \cite{chen2019missing}, we introduce a low-rank tensor completion model by incorporating the multiple truncated nuclear norm for rank approximation, which reinforces the commonly-used nuclear norm.
    % \item Demonstrating the superiority of the model with extensive data sets and experiments. We investigate the problem of missing data imputation thoroughly by extensive data sets and experiments. By comparing to the state-of-the-art model, the results demonstrate that the proposed model can enable us to get more accurate imputation results with higher efficiency.
\end{itemize}

The remainder of this paper is organized as follows. In Section~\ref{sec:method}, we first present the basic formulation of low-rank tensor completion (LRTC), and then we propose a new formulation based on truncated nuclear norm (TNN) minimization by introducing a universal rate parameter to control the degree of truncation. Section~\ref{sec:implementation} depicts the main pipeline for model implementation and application. In Section~\ref{sec:experiments}, we conduct numerical experiments on several real-world data sets under two aforementioned missing scenarios, and we evaluate the proposed LRTC-TNN model against several state-of-the-art baseline models. Section~\ref{sec:conclusion} summarizes this study and we discuss several directions for future research.

% The advantages of the

% The advantage is clear: applying a low-rank tensor completion has affordable computational cost and is scalable to large problem size when performing iterative optimization problem.

% It is often unclear whether an optimization algorithm can converge to the desired global solution and, if so, how fast this can be accomplished.

% \begin{itemize}
%     \item \textbf{Data}. What is an appropriate form of spatiotemporal data? In the previous studies \citep{chen2018spatial,chen2019missing,chen2019abayesian}, we could see that the spatiotemporal data collected in transportation systems are naturally tensors, and they can be easily structured into a tensor by their spatial locations and temporal attributes.

%     \item \textbf{Model}. Which type of machine learning model is an elegant and compact solver for the missing traffic data imputation? Be not only effective but also efficient? Why is it reasonable for each specific data set?

%     \item \textbf{Imputation result}. How does the model perform? This requires us to perform accurate results using any type of models.

% \end{itemize}

% We should make sure that the given missing data imputation problem is not ill-posed, and it must be consistent with the low-rank assumption. This requirement is not difficult to satisfy!

\section{Methodology} \label{sec:method}

In this section, we first introduce the basic formulation of the LRTC model. Then, we introduce the LRTC with truncated nuclear norm (TNN) minimization in detail. In particular, we define the formulation of TNN on tensors and integrate it into LRTC.

%\subsection{Preliminaries}

% In this section, we mainly introduce the foundation of tensor computations and the standard formula of low-rank tensor completion (LRTC).

\subsection{Notations}

Throughout this work, we follow the same notations as in \cite{kolda2009tensor}. We use boldface uppercase letters to denote matrices, e.g., $\boldsymbol{X}\in\mathbb{R}^{m\times n}$, boldface lowercase letters to denote vectors, e.g., $\boldsymbol{x}_{i}\in\mathbb{R}^{n}$, and lowercase letters to denote scalars, e.g., $x_{ij}$. Given a matrix $\boldsymbol{X}$, the Frobenius norm is defined as $\|\boldsymbol{X}\|_{F}=\sqrt{\sum_{i,j}x_{ij}^{2}}$. We denote a $d$th-order tensor and its entries by $\boldsymbol{\mathcal{X}}\in\mathbb{R}^{n_1\times \cdots \times n_d}$ and $x_{i_1...i_d}$, respectively. We also define the Frobenius norm on tensor as $\|\boldsymbol{\mathcal{X}}\|_F=\sqrt{\sum_{i_1,i_2,...,i_d}x_{i_1i_2\cdots i_d}}$. Let $\boldsymbol{\mathcal{X}}_{(k)} \in \mathbb{R}^{n_{k} \times\left(\prod_{l \neq k} n_{l}\right)}$ denote the $k$th-mode unfolding of tensor $\boldsymbol{\mathcal{X}}$ for $k=1,...,d$. Correspondingly, we define a folding operator that converts a matrix to a higher-order tensor in the $k$th-mode as $\text{fold}_{k}(\cdot)$. Thus, we have $\text{fold}_{k}(\boldsymbol{\mathcal{X}}_{(k)})=\boldsymbol{\mathcal{X}}$.

\subsection{Low-Rank Tensor Completion (LRTC)}

LRTC is a family of tensor completion techniques. This type of machine learning model is essentially built on the low-rank assumption on the partially observed input tensor, which is same to the low-rank matrix completion. In this work, we mainly focus on modeling a third-order tensor in the context of spatiotemporal traffic data. For a partially observed third-order tensor $\mathcal{Y}\in\mathbb{R}^{M\times N\times T}$, the LRTC model can be formulated as
\begin{equation}
    \begin{aligned}
    &\min_{\boldsymbol{\mathcal{X}}}~\operatorname{rank}(\boldsymbol{\mathcal{X}}) \\
    &\text{s.t.}~\mathcal{P}_{\Omega}(\boldsymbol{\mathcal{X}})=\mathcal{P}_{\Omega}(\boldsymbol{\mathcal{Y}}),
    \end{aligned}
    \label{lrtc}
\end{equation}
where $\boldsymbol{\mathcal{X}}\in\mathbb{R}^{M\times N\times T}$ is the recovered tensor that we hope to find, $\Omega$ is the index set of the observed entries \citep{liu2013tensor}. The notation $\text{rank}(\cdot)$ refers to the algebraic rank, and extension to higher-order tensors is straightforward. The operator $\mathcal{P}_{\Omega}:\mathbb{R}^{M \times N\times T} \mapsto \mathbb{R}^{M \times N\times T}$ is an orthogonal projection supported on $\Omega$:
\begin{equation*}
    \left[\mathcal{P}_{\Omega}\left(\boldsymbol{\mathcal{X}}\right)\right]_{mnt}=\left\{
    \begin{array}{ll}
    x_{mnt}, &\text{if} \left(m,n,t\right)\in\Omega, \\
    0, & \text{otherwise},
    \end{array}\right.
\end{equation*}
for any tensor $\boldsymbol{\mathcal{X}}$. The operator $\mathcal{P}_{\Omega}^{\perp}: \mathbb{R}^{M\times N\times T}\mapsto\mathbb{R}^{M\times N\times T}$ denotes the projection onto the complementary set of $\Omega$. The relationship between these two operators is $\mathcal{P}_{\Omega}(\boldsymbol{\mathcal{X}})+\mathcal{P}_{\Omega}^{\perp}(\boldsymbol{\mathcal{X}})=\boldsymbol{\mathcal{X}}$.

The rank minimization problem in Eq.~\eqref{lrtc} for tensor completion is NP-hard and computationally intractable \citep{liu2013tensor}. To deal with this issue, recent research focuses on identifying a possible convex relaxation as an alternative to the rank minimization problem. For example, \cite{liu2013tensor} proposed to use the multiple NN to replace the rank function in the objective function:
\begin{equation}
    \begin{aligned}
    &\min_{\boldsymbol{\mathcal{X}}}~\sum_{k=1}^{3}\alpha_k\|\boldsymbol{\mathcal{X}}_{(k)}\|_{*} \\
    &\text { s.t.}~\mathcal{P}_{\Omega}\left(\boldsymbol{\mathcal{X}}\right)=\mathcal{P}_{\Omega}\left(\boldsymbol{\mathcal{Y}}\right),
    \end{aligned}
    \label{LRTC_NN}
\end{equation}
where $\alpha_k\geq 0$ for $k=1,2,3$ are weight parameters. In this objective function, for any matrix $\boldsymbol{X}$, the NN is defined as $\|\boldsymbol{{X}}\|_{*}=\sum_{i}\sigma_{i}(\boldsymbol{{X}})$ and $\sigma_{i}(\boldsymbol{{X}})$ is the $i$th biggest singular value of $\boldsymbol{{X}}$. In an early study, \cite{fazel2002matrix} proved that the standard NN $\|\boldsymbol{X}\|_{*}$ is the convex envelope of the nonconvex rank function, i.e., $\text{rank}(\boldsymbol{X})$, when the largest singular value is not larger than 1.

Despite that the NN minimization method has achieved tremendous progress and remarkable success in  missing matrix/tensor data imputation tasks in previous studies, recent studies suggest that the result can be significantly improved by using certain nonconvex functions constructed on singular values instead of the default NN \citep{hu2013fast,gu2014weighted,yao2019largescale}. For example, instead of minimizing all singular values of a matrix simultaneously, TNN minimization keeps large singular values unchanged to preserve major components and only consider those small singular values as variables. In this work, our preliminary goal is to define a type of TNN on tensors to replace the corresponding NN in Eq.~\eqref{LRTC_NN}.

% \subsection{Nonconvex Regularization}

% Seeking global optimal solution to nonconvex regularization minimization problem.

% In this study, we apply truncated nuclear norm regularization to our tensor completion problem. Our consideration towards this choice is twofold: (1) truncated nuclear  norm is more easy to find a suboptimal solution than other regularization, and (2) the ...

% What examples are there?

% Has it advantages or virtues?

\subsection{LRTC with TNN Minimization (LRTC-TNN)}

% Although the nuclear norm regularization is a rather reasonable surrogate for rank function.

%\subsubsection{TNN on Tensors}

Before introducing the formulation of TNN on tensors, we first give a commonly used definition for the TNN on matrices.

% \newdefinition{rmk}{Definition}

\begin{definition}[Truncated Nuclear Norm on Matrices \citep{zhang2012matrix,hu2013fast}] Given a matrix $\boldsymbol{X}\in\mathbb{R}^{m\times n}$ and a positive integer $r<\min\{m,n\}$, TNN $\|\boldsymbol{X}\|_{r,*}$ is defined as the sum of $\min\{m,n\}-r$ minimum singular values, i.e.,
\begin{equation}
\|\boldsymbol{X}\|_{r,*}=\sum_{i=r+1}^{\min\{m,n\}}\sigma_{i}(\boldsymbol{X}),
\label{tnn}
\end{equation}
where $\sigma_{i}(\boldsymbol{X})$ is the $i$th singular value of $\boldsymbol{X}$. The singular values are sorted as $\sigma_{1} \geq \sigma_{2} \geq \cdots \geq \sigma_{\min \{m, n\}} \geq 0$.
\label{def1}
\end{definition}

In Definition~\ref{def1}, the largest $r$ singular values $\sigma_1(\boldsymbol{X}),...,\sigma_{r}(\boldsymbol{X})$ do not contribute to the TNN function. This truncation definition makes it possible to minimize the NN subtracted by the sum of the largest $r$ singular values rather than all singular values simultaneously. However, this definition on matrices cannot be directly used for multiway/multidimensional tensors. In accordance with TNN on matrices, we follow \citet{liu2013tensor} and propose to use tensor unfoldings (i.e., matrices) to define the TNN on tensors.

\begin{definition}[Truncated Nuclear Norm on Tensors]
For any $d$th-order tensor $\boldsymbol{\mathcal{X}}\in\mathbb{R}^{n_1\times n_2\times \cdots\times n_d}$, the multiple TNN is defined as
\begin{equation}
\|\boldsymbol{\mathcal{X}}\|_{\theta,*}=\sum_{k=1}^{d}\alpha_k\|\boldsymbol{\mathcal{X}}_{(k)}\|_{r_{k},*},
\label{tensor_tnn}
\end{equation}
with the truncation for each tensor mode being
\begin{equation}
r_{k}=\lceil \theta \cdot\min\{n_k,\prod_{h\neq k}n_{h}\}\rceil,\forall k\in\{1,2,...,d\},
\end{equation}
where $\left\lceil \cdot\right\rceil$ means the smallest integer that is not less than the given value, $\theta$ is an universal rate parameter which controls the whole truncation on $d$ modes of $\boldsymbol{\mathcal{X}}$ and it should satisfies $1\leq r_k<\min\{n_k,\prod_{h\neq k}n_{h}\}$ in this case, $\alpha_1,...,\alpha_d$ with $(\sum_{k}\alpha_k=1)$ are weight parameters imposed on all the TNN of unfolding matrices $\boldsymbol{\mathcal{X}}_{(1)},...,\boldsymbol{\mathcal{X}}_{(d)}$, respectively.
\label{def2}
\end{definition}

Provided Definition~\ref{def2} as the basis of LRTC with TNN minimization, we believe that if the rate parameter $\theta$ can be set appropriately, the truncation for each tensor mode would be assigned automatically. By replacing the NN in Eq.~\eqref{LRTC_NN} with the defined TNN in Eq.~\eqref{tensor_tnn}, the LRTC model is now given by
\begin{equation}
    \begin{aligned}
    &\min_{\boldsymbol{\mathcal{X}}}~\sum_{k=1}^{3}\alpha_k\|\boldsymbol{\mathcal{X}}_{(k)}\|_{r_k,*} \\
    &\text { s.t.}~\mathcal{P}_{\Omega}\left(\boldsymbol{\mathcal{X}}\right)=\mathcal{P}_{\Omega}\left(\boldsymbol{\mathcal{Y}}\right).
    \end{aligned}
    \label{LRTC_TNN}
\end{equation}

However, Eq.~\eqref{LRTC_TNN} is still not in its appropriate form. In fact, unfolding a tensor in different mode cannot guarantee the dependencies of variables (i.e., unfolding matrices of the tensor) in the objective function \citep{liu2013tensor}. To this end, we introduce an auxiliary tensor variable $\boldsymbol{\mathcal{M}}$ and an additional set of constraints $\boldsymbol{\mathcal{X}}_{k}=\boldsymbol{\mathcal{M}},k=1,2,3$ to convert the problem \eqref{LRTC_TNN} into a tractable problem as follows,
\begin{equation}
    \begin{aligned}
    &\min \limits_{\boldsymbol{\mathcal{M}},\boldsymbol{\mathcal{X}}_{1},\boldsymbol{\mathcal{X}}_{2},\boldsymbol{\mathcal{X}}_{3}}~\sum_{k=1}^{3}\alpha_k\|\boldsymbol{\mathcal{X}}_{k(k)}\|_{r_k,*} \\
    &\text { s.t.}~\left\{\begin{array}{l} \boldsymbol{\mathcal{X}}_{k}=\boldsymbol{\mathcal{M}}, k=1,2,3,\\
    \mathcal{P}_{\Omega}\left(\boldsymbol{\mathcal{M}}\right)=\mathcal{P}_{\Omega}\left(\boldsymbol{\mathcal{Y}}\right), \\ \end{array}\right.
    \end{aligned}
    \label{LRTC_TNN_new}
\end{equation}
where $\boldsymbol{\mathcal{M}}$ is introduced to keep observation information and then broadcast such information to the variables $\boldsymbol{\mathcal{X}}_{k},k=1,2,3$. In the current formulation, tensors $\boldsymbol{\mathcal{X}}_{k},k=1,2,3$ are involved with TNNs while another variable $\boldsymbol{\mathcal{M}}$ establishes the relationship with the partially observed tensor $\boldsymbol{\mathcal{Y}}$.

\subsection{Solution Algorithms}

To solve this optimization problem, a straightforward and widely used approach is the Alternating Direction Method of Multipliers (ADMM) framework. This framework has been experimentally proved to be efficient for low-rank matrix/tensor completion models \citep{hu2013fast,liu2013tensor}. Before setting up an ADMM model, we need to write the augmented Lagrangian function of Eq.~\eqref{LRTC_TNN_new} in advance:
\begin{equation}
    \begin{aligned}
    \mathcal{L}(\boldsymbol{\mathcal{M}},\left\{\boldsymbol{\mathcal{X}}_{k},\boldsymbol{\mathcal{T}}_{k}\right\}_{k=1}^{3})
    =\sum_{k=1}^{3}\left(\alpha_k\|\boldsymbol{\mathcal{X}}_{k(k)}\|_{r_k,*}+\frac{\rho_k}{2}\|\boldsymbol{\mathcal{X}}_{k(k)}-\boldsymbol{\mathcal{M}}_{(k)}\|_{F}^{2}+\big\langle\boldsymbol{\mathcal{X}}_{k}-\boldsymbol{\mathcal{M}},\boldsymbol{\mathcal{T}}_{k}\big\rangle\right), \\
    \end{aligned}
    \label{lagrangian}
\end{equation}
where $\big\langle\cdot,\cdot\big\rangle$ denotes the inner product. The auxiliary variables $\boldsymbol{\mathcal{T}}_{1},\boldsymbol{\mathcal{T}}_{2},\boldsymbol{\mathcal{T}}_{3}\in\mathbb{R}^{M\times N\times T}$ are used for dual update in the following ADMM scheme.

Accordingly, ADMM transforms the original tensor completion problem to the following three subproblems in an iterative manner:
\begin{equation}
    \begin{aligned}
    \boldsymbol{\mathcal{X}}_{k}^{l+1}:&=\operatorname{arg}\min_{\boldsymbol{\mathcal{X}}}\mathcal{L}(\boldsymbol{\mathcal{M}},\{\boldsymbol{\mathcal{X}}_{k}^{l+1},\boldsymbol{\mathcal{T}}_{k}^{l}\}_{k=1}^{3}), \\
    \boldsymbol{\mathcal{M}}^{l+1}:&=\operatorname{arg}\min_{\boldsymbol{\mathcal{M}}}\mathcal{L}(\boldsymbol{\mathcal{M}},\{\boldsymbol{\mathcal{X}}_{k}^{l+1},\boldsymbol{\mathcal{T}}_{k}^{l}\}_{k=1}^{3}), \\
    \boldsymbol{\mathcal{T}}_{k}^{l+1}:&=\boldsymbol{\mathcal{T}}_{k}^{l}+\rho_k(\boldsymbol{\mathcal{X}}^{l+1}_{k}-\boldsymbol{\mathcal{M}}^{l+1}), \\
    \end{aligned}
    \label{admm_prob}
\end{equation}
where we follow the order $\boldsymbol{\mathcal{X}}_{1}^{l+1}\to\cdots\to\boldsymbol{\mathcal{X}}_{3}^{l+1}\to\boldsymbol{\mathcal{M}}^{l+1}\to\boldsymbol{\mathcal{T}}_{1}^{l+1}\to\cdots\to\boldsymbol{\mathcal{T}}_{3}^{l+1}$.

In practice, if $\rho_1=\rho_2=\rho_3=\rho$, the variables $\boldsymbol{\mathcal{T}}_{k}^{l+1}$ can be updated by
\begin{equation}
    \boldsymbol{\mathcal{T}}^{l+1}:=\boldsymbol{\mathcal{T}}^{l}+\rho(\tilde{\boldsymbol{\mathcal{X}}}^{l+1}-\tilde{\boldsymbol{\mathcal{M}}}^{l+1}),
    \label{update_variable_T}
\end{equation}
where $\boldsymbol{\mathcal{T}},\tilde{\boldsymbol{\mathcal{X}}},\tilde{\boldsymbol{\mathcal{M}}}$ are fourth-order tensors of the same size $M\times N\times T\times 3$. $\boldsymbol{\mathcal{T}}$ and $\tilde{\boldsymbol{\mathcal{X}}}$ are stacked by third-order tensors $\boldsymbol{\mathcal{T}}_{k}$s and ${\boldsymbol{\mathcal{X}}}_{k}$s over the fourth mode, respectively, and $\tilde{\boldsymbol{\mathcal{M}}}$ is stacked over the fourth mode by copying the third-order tensor $\boldsymbol{\mathcal{M}}$.

% {\color{red}
% In particular, the subproblems in Eq.~\eqref{admm_prob} for computing $\boldsymbol{\mathcal{X}}_k$s and $\boldsymbol{\mathcal{M}}$ admit closed-form solutions. We would discuss these two problems in the following for detail.
% }

\subsubsection{Computing Tensors $\boldsymbol{\mathcal{X}}_{k},k=1,2,3$}

Based on Eqs.~\eqref{lagrangian} and \eqref{admm_prob}, the subproblem for getting the optimal $\boldsymbol{\mathcal{X}}_{k}^{l+1}$ is
\begin{equation}
    \begin{aligned}
    \boldsymbol{\mathcal{X}}_{k}^{l+1}:&=\operatorname{arg}\min_{\boldsymbol{\mathcal{X}}}~\alpha_k\|\boldsymbol{\mathcal{X}}_{(k)}\|_{r_k,*}+\frac{\rho_k}{2}\|\boldsymbol{\mathcal{X}}_{(k)}-\boldsymbol{\mathcal{M}}_{(k)}^{l}\|_{F}^{2}+\big\langle\boldsymbol{\mathcal{X}}_{(k)},\boldsymbol{\mathcal{T}}_{k(k)}^{l}\big\rangle \\
    &=\operatorname{arg}\min_{\boldsymbol{\mathcal{X}}}~\alpha_k\|\boldsymbol{\mathcal{X}}_{(k)}\|_{r_k,*}+\frac{\rho_k}{2}\|\boldsymbol{\mathcal{X}}_{(k)}\|_{F}^{2}-\rho_k\big\langle\boldsymbol{\mathcal{X}}_{(k)},\boldsymbol{\mathcal{M}}_{(k)}^{l}-\frac{1}{\rho_k}\boldsymbol{\mathcal{T}}_{k(k)}^{l}\big\rangle, \\
    &=\operatorname{arg}\min_{\boldsymbol{\mathcal{X}}}~\alpha_k\|\boldsymbol{\mathcal{X}}_{(k)}\|_{r_k,*}+\frac{\rho_k}{2}\|\boldsymbol{\mathcal{X}}_{(k)}-(\boldsymbol{\mathcal{M}}_{(k)}^{l}-\frac{1}{\rho_k}\boldsymbol{\mathcal{T}}_{k(k)}^{l})\|_{F}^{2}, \\
    & =\operatorname{arg}\min_{\boldsymbol{\mathcal{X}}}~ G(\boldsymbol{\mathcal{X}}_{(k)}).
    \end{aligned}
    \label{optimal_x}
\end{equation}
% The objective function of Eq.~\eqref{optimal_x} can be written as
% \begin{equation}
%     G(\boldsymbol{\mathcal{X}}_{(k)})=\alpha_k\|\boldsymbol{\mathcal{X}}_{(k)}\|_{r_k,*}+\frac{\rho_k}{2}\|\boldsymbol{\mathcal{X}}_{(k)}-(\boldsymbol{\mathcal{M}}_{(k)}^{l}-\frac{1}{\rho_k}\boldsymbol{\mathcal{T}}_{k(k)}^{l})\|_{F}^{2},
%     \label{func_x}
% \end{equation}

In general, the optimal $\boldsymbol{\mathcal{X}}$ can be found by solving $\boldsymbol{0}\in\partial G(\boldsymbol{\mathcal{X}}_{(k)})$ when $G(\boldsymbol{\mathcal{X}}_{(k)})$ is convex. However, for a nonconvex $G(\cdot)$, a global optimal solution is not guaranteed \citep{hu2013fast} and it remains unclear whether an optimization problem like Eq.~\eqref{optimal_x} can converge to the desired global optima. In this work, we adopt a closed-form optimal/sub-optimal solution to this problem by using the following theorem.

\begin{theorem}
For any $\alpha,\rho>0$, $\boldsymbol{Z}\in\mathbb{R}^{m\times n}$, and $r\in\mathbb{N}_{+}$ where $r<\min\{m,n\}$, an optimal solution to the problem
\begin{equation}
    \min_{\boldsymbol{X}}~\alpha\|\boldsymbol{X}\|_{r,*}+\frac{\rho}{2}\|\boldsymbol{X}-\boldsymbol{Z}\|_{F}^{2},
    \label{tnn_prob}
\end{equation}
is given by the generalized singular value thresholding
\begin{equation}
    \hat{\boldsymbol{X}}=\boldsymbol{U}\boldsymbol{\Sigma}_{\alpha/\rho}\boldsymbol{V}^\top,
    \label{gsvt}
\end{equation}
where $\boldsymbol{U}\boldsymbol{\Sigma}\boldsymbol{V}^{\top}$ is the singular value decomposition of $\boldsymbol{Z}$. The shrinkage of singular values is therefore defined as
\begin{equation}
    \boldsymbol{\Sigma}_{\alpha/\rho}=\text{diag}\left((\sigma_{1},\cdots,\sigma_{r},[\sigma_{r+1}-\alpha/\rho]_{+},\cdots,[\sigma_{\min\{m,n\}}-\alpha/\rho]_{+})^\top\right),
    \label{shrinkage}
\end{equation}
where $\sigma_{1},...,\sigma_{\min\{m,n\}}$ are diagonal entries of $\boldsymbol{\Sigma}$, and $[\cdot]_{+}$ denotes the positive truncation at 0 which satisfies $[\sigma-\alpha/\rho]_{+}=\max\{\sigma-\alpha/\rho,0\}$.
\label{thm1}
\end{theorem}

\begin{proof}
For any $x,z>0$, define the function $f_i(\cdot)$ as
\begin{equation}
    f_i(x)=\left\{\begin{array}{ll}
        \alpha x+\frac{\rho}{2}(x-z)^2, & \text{if $i>r$,} \\
        \frac{\rho}{2}(x-z)^2,   & \text{otherwise}, \\
    \end{array}\right.
    \label{g_fun}
\end{equation}
which is differentiable, and
\begin{equation}
    \frac{\partial f_i(x)}{\partial x}=\left\{\begin{array}{ll}
        \alpha +\rho(x-z), & \text{if $i>r$,} \\
        \rho(x-z),   & \text{otherwise}. \\
    \end{array}\right.
\end{equation}

When minimizing the objective function $F(\boldsymbol{X})=\alpha\|\boldsymbol{X}\|_{r,*}+\frac{\rho}{2}\|\boldsymbol{X}-\boldsymbol{Z}\|_{F}^{2}$, we should find $\boldsymbol{X}$ such that $\boldsymbol{0}\in\partial F(\boldsymbol{X})$ to optimize $F$. To solve $\boldsymbol{0}\in\partial F(\boldsymbol{X})$, we construct the singular value function referring to \cite{larsson2014rank} as $f:\mathbb{R}^{\min\{m,n\}}\mapsto\mathbb{R}$ as $f(\boldsymbol{\sigma}(\boldsymbol{X}))=\sum_{i=1}^{\min\{m,n\}}f_{i}(\sigma_i(\boldsymbol{X}))$ where the function $f_{i}(\cdot)$ is defined as Eq.~\eqref{g_fun}. Considering the components of the sum in $f(\boldsymbol{\sigma}(\boldsymbol{X}))$ separately, we have
\begin{equation}
    0\in \frac{\partial f_i(\sigma_i(\boldsymbol{X}))}{\partial \sigma_i(\boldsymbol{X})}\quad\Rightarrow\quad \sigma_i(\boldsymbol{X})=\left\{\begin{array}{ll}
        \left[\sigma_i(\boldsymbol{Z})-\alpha/\rho\right]_+, & \text{if $i>r$,} \\
        \sigma_i(\boldsymbol{Z}),   & \text{otherwise}, \\
    \end{array}\right.
\end{equation}
and this result is same as Eq.~\eqref{shrinkage}.

\end{proof}

On the other hand, the TNN minimization problem described in Eq.~\eqref{tnn_prob} is indeed a special case of the weighted NN minimization problem, i.e.,
\begin{equation*}
    \min_{\boldsymbol{X}}~\alpha\boldsymbol{w}^\top\boldsymbol{\sigma}(\boldsymbol{X})+\frac{\rho}{2}\|\boldsymbol{X}-\boldsymbol{Z}\|_{F}^{2},
    \label{wnn_prob}
\end{equation*}
with the weights being $w_{1},...,w_{r}=0$ and $w_{r+1},...,w_{\min\{m,n\}}=1$. As proved by \cite{zhang2011penalty} (Proposition 2.1) and \cite{chen2013reduced} (Theorem 2.3), in the situation of $0\leq w_1\leq\cdots\leq w_{\min\{m,n\}}$ (i.e., order constraint), weighted NN minimization problem has an optimal solution and the shrinkage of singular values corresponding to Eq.~\eqref{shrinkage} is given by
\begin{equation}
    \boldsymbol{\Sigma}_{\alpha/\rho}=\text{diag}\left(([\sigma_{1}-\alpha w_1/\rho]_{+},\cdots,[\sigma_{\min\{m,n\}}-\alpha w_{\min\{m,n\}}/\rho]_{+})^\top\right),
    \label{weighted_shrinkage}
\end{equation}
which is consistent with Theorem~\ref{thm1}.

Now, according to Theorem \ref{thm1}, the shrinkage of singular values for the problem \eqref{optimal_x} is
\begin{equation}
    \sigma_i(\boldsymbol{\mathcal{X}}_{(k)})=\left\{\begin{array}{ll}
        \left[\sigma_{i}(\boldsymbol{\mathcal{M}}_{(k)}^{l}-\frac{1}{\rho_k}\boldsymbol{\mathcal{T}}_{k(k)}^{l})-\frac{\alpha_k}{\rho_k}\right]_{+}, & \text{if $i>r_k$,} \\
        \sigma_{i}(\boldsymbol{\mathcal{M}}_{(k)}^{l}-\frac{1}{\rho_k}\boldsymbol{\mathcal{T}}_{k(k)}^{l}), & \text{otherwise,}
    \end{array}\right.
\end{equation}
where the singular value decomposition of $\boldsymbol{\mathcal{M}}_{k(k)}^{l}-\frac{1}{\rho_k}\boldsymbol{\mathcal{T}}_{k(k)}^{l}$ is $\boldsymbol{U}\boldsymbol{\Sigma} \boldsymbol{V}^\top$ and the diagonal entries of $\boldsymbol{\Sigma}$ are $\sigma_{i}(\boldsymbol{\mathcal{M}}_{k(k)}^{l}-\frac{1}{\rho_k}\boldsymbol{\mathcal{T}}_{k(k)}^{l})$s. We could get an optimal $\boldsymbol{\mathcal{X}}_{k}^{l+1}$ by
\begin{equation}
    \boldsymbol{\mathcal{X}}_{k}^{l+1}=\text{fold}_{k}\left(\boldsymbol{U}\text{diag}(\boldsymbol{\sigma}(\boldsymbol{\mathcal{X}}_{(k)}))\boldsymbol{V}^\top\right).
    \label{update_x_func}
\end{equation}

\subsubsection{Computing Tensor $\boldsymbol{\mathcal{M}}$}

In terms of $\boldsymbol{\mathcal{M}}^{l+1}$, there is a least square solution based on  Eq.~\eqref{lagrangian} and \eqref{admm_prob}, which is given by
\begin{equation}
    \begin{aligned}
    \boldsymbol{\mathcal{M}}^{l+1}:&=\operatorname{arg}\min_{\boldsymbol{\mathcal{M}}}~\sum_{k=1}^{3}\left(\frac{\rho_k}{2}\|\boldsymbol{\mathcal{X}}_{k(k)}^{l+1}-\boldsymbol{\mathcal{M}}_{(k)}\|_{F}^{2}-\big\langle\boldsymbol{\mathcal{M}}_{(k)},\boldsymbol{\mathcal{T}}_{k(k)}^{l}\big\rangle\right) \\
    &=\operatorname{arg}\min_{\boldsymbol{\mathcal{M}}}~\sum_{k=1}^{3}\left(\frac{\rho_k}{2}\big\langle\boldsymbol{\mathcal{X}}_{k}^{l+1}-\boldsymbol{\mathcal{M}},\boldsymbol{\mathcal{X}}_{k}^{l+1}-\boldsymbol{\mathcal{M}}\big\rangle-\big\langle\boldsymbol{\mathcal{M}},\boldsymbol{\mathcal{T}}_{k}^{l}\big\rangle\right) \\
    &=\operatorname{arg}\min_{\boldsymbol{\mathcal{M}}}\sum_{k=1}^{3}\left(\frac{\rho_k}{2}\big\langle\boldsymbol{\mathcal{M}},\boldsymbol{\mathcal{M}}\big\rangle-\rho_k\big\langle\boldsymbol{\mathcal{M}},\boldsymbol{\mathcal{X}}_{k}^{l+1}\big\rangle-\big\langle\boldsymbol{\mathcal{M}},\boldsymbol{\mathcal{T}}_{k}^{l}\big\rangle\right) \\
    &=\operatorname{arg}\min_{\boldsymbol{\mathcal{M}}}\big\langle\boldsymbol{\mathcal{M}},\boldsymbol{\mathcal{M}}\big\rangle-\frac{2}{\sum_{k=1}^{3}\rho_k}\Big\langle\boldsymbol{\mathcal{M}},\sum_{k=1}^{3}\left(\rho_k\boldsymbol{\mathcal{X}}_{k}^{l+1}+\boldsymbol{\mathcal{T}}_{k}^{l}\right)\Big\rangle \\
    &=\frac{1}{\sum_{k=1}^{3}\rho_k}\sum_{k=1}^{3}\left(\rho_k\boldsymbol{\mathcal{X}}_{k}^{l+1}+\boldsymbol{\mathcal{T}}_{k}^{l}\right), \\
    \end{aligned}
    \label{optimal_M_variable}
\end{equation}
where we impose a fixed constraint, i.e., $\mathcal{P}_{\Omega}(\boldsymbol{\mathcal{M}}^{l+1}):=\mathcal{P}_{\Omega}(\boldsymbol{\mathcal{Y}})$, to guarantee the transformation of observation information at each iteration.

% then,
% \begin{equation}
%     \mathcal{P}_{\Omega}^\perp(\boldsymbol{\mathcal{Z}}_{k}^{l+1}) := \mathcal{P}_{\Omega}^{\perp}(\boldsymbol{\mathcal{X}}_{k}^{l+1}+\frac{1}{\rho_k}\boldsymbol{\mathcal{T}}_{k}^{l})
%     \label{update_z}
% \end{equation}

% or
% {\color{red}
% \begin{equation}
%     \mathcal{P}_{\Omega}(\boldsymbol{\mathcal{Z}}^{l+1}) := \mathcal{P}_{\Omega}(\boldsymbol{\mathcal{Y}})+\mathcal{P}_{\Omega}^{\perp}(\frac{1}{3}\sum_{k=1}^{3}(\boldsymbol{\mathcal{X}}_{k}^{l+1}+\frac{1}{\rho_k}\boldsymbol{\mathcal{T}}_{k}^{l}))
%     \label{update_z}
% \end{equation}
% }

\section{Model Implementation} \label{sec:implementation}

% \subsection{Main Procedures}

As shown in Figure~\ref{lrtc_flow}, our goal is to impute the incomplete tensor as a complete one using the proposed LRTC-TNN model. The main procedures are: (1) adapting the raw data to a tensor structure with a specific rule, (2) imputing incomplete tensor data with our proposed model, and (3) filling in missing entries.

A practical issue in applying this tensor learning framework for spatiotemporal traffic data imputation is how to establish the rule of structuring spatiotemporal traffic data as a tensor. In this study, we assume that third-order tensor with dimensions corresponding to spatial location, day, and time interval is a typical algebraic data structure, which is experimentally proved to be efficient on a number of spatiotemporal traffic data \citep{chen2019abayesian}. To find a solution to the problem \eqref{LRTC_TNN_new}, it is sufficient to carry out the following steps (see the middle panel of Figure~\ref{lrtc_flow}): (1) Initialize the variables and set some necessary parameters; (2) Update the variables $\boldsymbol{\mathcal{X}}_{k}^{l+1}$, $\boldsymbol{\mathcal{M}}^{l+1}$, and $\boldsymbol{\mathcal{T}}^{l+1}$ according to Eq.~\eqref{update_x_func}, Eq.~\eqref{optimal_M_variable}, and Eq.~\eqref{update_variable_T}, respectively.

\begin{figure}
\centering
  \includegraphics[width=\textwidth]{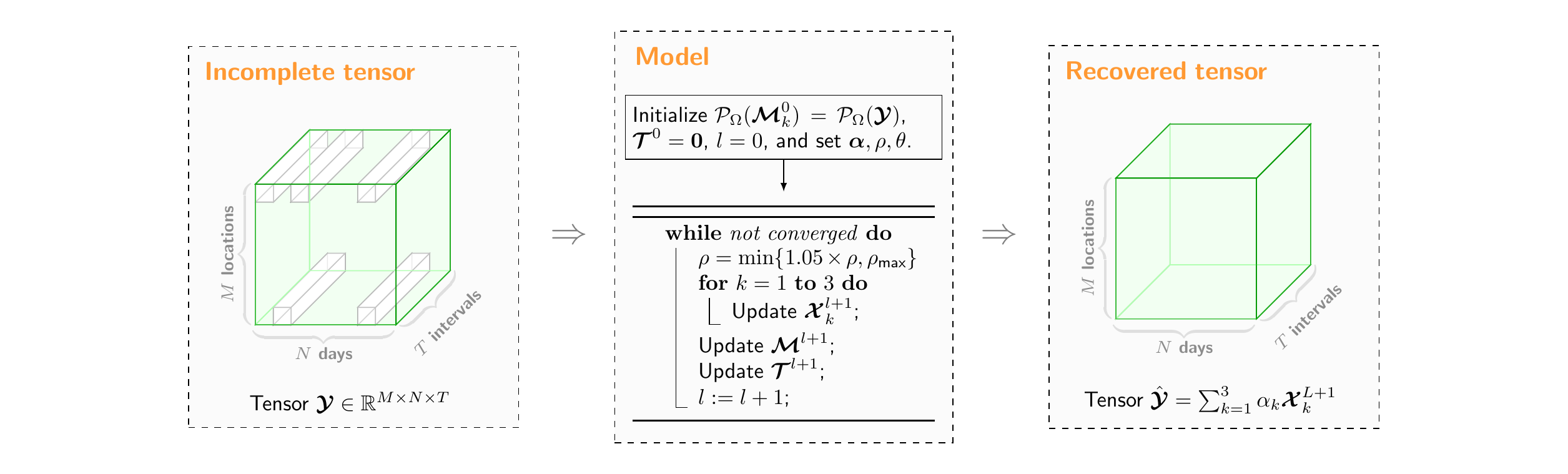}
\caption{Main procedures for a real-world data imputation task using the proposed LRTC-TNN model. Left panel: Organizing an incomplete tensor from partially observed traffic measurements; Middle panel: Implementing the proposed LRTC-TNN model; Right panel: Getting a recovered tensor.}
\label{lrtc_flow}
\end{figure}

% {\color{red}
% \subsection{Scalability of LRTC-TNN}

% In general, low-rank matrix/tensor completion models suffer from the high computational cost of to compute singular value decomposition for (generalized) singular value thresholding at each iteration \citep{oh2015fast}. In our case, it is the GSVT for updating $\boldsymbol{\mathcal{X}}_{k}$s (see Eq.~\eqref{update_x_func}). To avoid the direct computation of SVD, there are few studies that focused on developing fast and randomized SVD algorithms for low-rank minimization problems that involve SVT/GSVT \citep{oh2015fast,yao2019largescale}. The way that ... can make the model scalable to large-scale problems. Therefore, If applying the proposed model to large-scale problems, we have an option available to speed up the process of GSVT.

% This model can be scaled to large-scale problems because we could use randomized SVD to reduce the computational cost of GSVT in every iteration.

% full rank?
% }

\section{Experiments} \label{sec:experiments}

In this section, we conduct extensive experiments on four real-world spatiotemporal traffic data sets to evaluate the proposed LRTC-TNN model.
% In particular, we answer the following critical questions:

% \begin{itemize}
%     \item Does LRTC-TNN really improve the imputation beyond its competing models and achieve state-of-the-art results?
%     \item Is LRTC-TNN scalable to the large-scale spatiotemporal traffic data imputation task?
% \end{itemize}

\subsection{Data Sets}

We choose four publicly available data sets collected from real-world transportation systems as our benchmark data sets. Table~\ref{datasets} provides a brief summary of these data sets. For simplicity, as shown in Table~\ref{datasets}, we refer to the four data sets as data set "G", "B", "H", and "S", respectively. Essentially, the four data sets follow the same tensor structure--``location/sensor$\times$day$\times$time of day'' and they can also be organized in a matrix structure of ``location/sensor$\times$time'' by stacking the time dimension.

\begin{table}[!ht]
\caption{Brief description of the four selected traffic data sets.}
\label{datasets}
\centering
\footnotesize
% \scriptsize
\begin{tabular}{lcccccccc}
\toprule
Data Set & Description & Data Size \\
\midrule
\parbox[t]{0.15\textwidth}{(\textbf{G}) Guangzhou urban traffic speed data\tablefootnote{\url{https://doi.org/10.5281/zenodo.1205229}}} & \parbox[t]{0.55\textwidth}{The data set contains average traffic speed collected from 214 road segments over two months (from August 1 to September 30, 2016) with a 10-minute resolution (144 time intervals per day) in Guangzhou, China. This data set contains 1.29\% missing values.} & \parbox[t]{0.20\textwidth}{Tensor: $214\times 61\times 144$,\newline matrix: $214\times 8784$.} \\
\midrule
\parbox[t]{0.15\textwidth}{(\textbf{B}) Birmingham parking data\tablefootnote{\url{https://archive.ics.uci.edu/ml/datasets/Parking+Birmingham}}} & \parbox[t]{0.55\textwidth}{This data set registers occupancy (i.e., number of parked vehicles) of 30 car parks in Birmingham City for every half an hour between 8:00 and 17:00 over more than two months (77 days from October 4, 2016 to December 19, 2016). This data set contains 14.89\% missing values.} & \parbox[t]{0.20\textwidth}{Tensor: $30\times 77\times 18$,\newline matrix: $30\times 1386$.} \\
\midrule
\parbox[t]{0.15\textwidth}{(\textbf{H}) Hangzhou metro passenger flow data\tablefootnote{\url{https://tianchi.aliyun.com/competition/entrance/231708/information}}} & \parbox[t]{0.55\textwidth}{This data set provides incoming passenger flow of 80 metro stations over 25 days (from January 1 to January 25, 2019) with a 10-minute resolution in Hangzhou, China. We discard the interval 0:00 a.m. – 6:00 a.m. with no services (i.e., only consider the remaining 108 time intervals of a day).} & \parbox[t]{0.20\textwidth}{Tensor: $80\times 25\times 108$,\newline matrix: $80\times 2700$.} \\
\midrule
\parbox[t]{0.15\textwidth}{(\textbf{S}) Seattle freeway traffic speed data\tablefootnote{\url{https://github.com/zhiyongc/Seattle-Loop-Data}}} & \parbox[t]{0.55\textwidth}{This data set contains freeway traffic speed from 323 loop detectors with a 5-minute resolution over the whole year of 2015 in Seattle, USA. We choose the subset in January (4 weeks from January 1 to January 28) as our experiment data.} & \parbox[t]{0.20\textwidth}{Tensor: $323\times 28\times 288$,\newline matrix: $323\times 8064$.} \\
\bottomrule
\end{tabular} \\
% \footnote{\url{https://doi.org/10.5281/zenodo.1205229}}
\end{table}

\subsection{Baseline Models}

We compare the proposed LRTC-TNN model with the following baseline models:
\begin{itemize}
    \item BTMF: Bayesian Temporal Matrix Factorization \citep{chen2019bayesian}. This is a fully Bayesian matrix factorization model which integrates vector autoregressive (VAR) model into the latent temporal factors. With the flexible VAR process, BTMF achieves superior accuracy over other matrix factorization models (without temporal modeling) and tensor factorization models in imputation tasks.
    \item TRMF: Temporal Regularized Matrix Factorization  \citep{yu2016temporal}. This is a temporal matrix factorization model which applies a multiple autoregressive (AR) process to model latent temporal factors. BTMF and TRMF can be considered a generalized version of the Bayesian temporal tensor factorization model by \citet{xiong2010temporal}, which imposes temporal smoothness with a dynamic model.
    \item BGCP: Bayesian Gaussian CP decomposition  \citep{chen2019abayesian}. This is a fully Bayesian tensor factorization model which uses Markov chain Monte Carlo (MCMC) to learn the latent factor matrices (i.e., low-rank structure).
    \item BATF: Bayesian Augmented Tensor Factorization  \citep{chen2019missing}. This is a fully Bayesian model built on a special tensor factorization formula, in which components include both explicit variables and latent factors (i.e., low-rank factors). Variables in this model are learned via variational Bayes.
    \item HaLRTC: High-accuracy Low-Rank Tensor Completion  \citep{liu2013tensor}. This is a LRTC model which uses NN minimization to find accurate estimation for unobserved/missing entries in tensor data.
\end{itemize}

%We have following considerations on choosing baseline models. In general, we take into account two types of imputation approaches built on the assumption of underlying low-rank structure, which has been proved to be competitively effective and efficient in the previous studies. The first is low-rank temporal matrix factorization approaches like BTMF and TRMF which impose auto-regression processes on latent temporal factors, and this type enables us to find more informative low-rank structures than the standard matrix factorization. Another type is using low-rank tensor structure in which the tensor factorization models like BGCP and BATF can be used to find a compact low-rank structure explicitly, while the LRTC model like HaLRTC can apply the information of nuclear norm of data tensors. These models have been experimentally proved to be state-of-the-art in the past literature.

We have the following considerations on choosing the baseline models. In general, we take into account two types of imputation approaches built on the assumption of underlying low-rank structure, which has been proved to be effective and efficient in the previous studies. The first is low-rank temporal matrix factorization models like BTMF and TRMF. These models follow the matrix structure (i.e., ``location/senor$\times$time'') and impose smoothness on latent temporal factors by introducing autoregressive/dynamic assumptions. Therefore, these temporal models produce more informative and consistent low-rank structures than those estimated from standard matrix factorization. The second approach follows the tensor structure (i.e., ``location/senor$\times$day$\times$time of day'') to better utilize the information across different days. The tensor approach includes tensor factorization models (BGCP and BATF) and a standard LRTC model (HaLRTC). In the context of spatiotemporal traffic data, these models can utilize the intrinsic characteristics of these data (e.g., daily similarity).

% Note that what we care about are both accuracy and efficiency.

\subsection{Experiment Setting}

For model evaluation, we mask certain entries of the data as missing values and then perform imputation for these ``missing'' entries. We use the actual values (ground truth) of these entries to compute the metrics MAPE and RMSE:
\begin{equation}
\operatorname{MAPE}=\frac{1}{n} \sum_{i=1}^{n} \left|\frac{y_{i}-\hat{y}_{i}}{y_{i}}\right| \times 100, \quad    \operatorname{RMSE}=\sqrt{\frac{1}{n} \sum_{i=1}^{n}\left(y_{i}-\hat{y}_{i}\right)^{2}}.
\end{equation}

Following \citep{chen2018spatial}, we design two missing patterns---random missing (RM, see Figure~\ref{missing}(a)) and non-random missing (NM, see Figure~\ref{missing}(b)). The NM scenario is more challenging since the data is corrupted in a correlated manner. These two missing scenarios can help us better assess the performance and effectiveness of different models.

There are some parameters to be set in LRTC-TNN, including $\alpha_k$ and $\theta$. Parameters $\alpha_1,\alpha_2,\alpha_3$ capture the importance/weight of the three tensor unfoldings. When applying the model, it is feasible to give different weights on different unfoldings. However, this requires prior knowledge that we do not have. Although it is feasible to tune these parameters using cross validation, the process will be computationally very expensive as the procedure has to be performed for each data sets in each missing scenarios. Instead of tuning these parameters, we simply use the same weights ($\alpha_1=\alpha_2=\alpha_3=\frac{1}{3}$) for the unfoldings following the tensor nuclear norm and HaLRTC in \citet{liu2013tensor}. We update $\rho$ in each iteration with $\rho=\min \{ 1.05\times \rho,\rho_{\max}  \} $ as in \citet{lin2011linearized} and set the initial value of $\rho$ to be $\rho=10^{-5}$ and the maximum value as $\rho_{\max}=10^5$. We use $\mathcal{C}^{l+1}=\frac{\|\hat{\boldsymbol{\mathcal{Y}}}^{l+1}-\hat{\boldsymbol{\mathcal{Y}}}^{l}\|_{F}}{\|\mathcal{P}_{\Omega}(\boldsymbol{\mathcal{Y}})\|_{F}} < \epsilon$ as convergence condition, where $\hat{\boldsymbol{\mathcal{Y}}}^{l+1}$ and $\hat{\boldsymbol{\mathcal{Y}}}^{l}$ denote the recovered tensors at the $l+1$th iteration and the $l$th iteration, respectively. We set $\epsilon=1\times 10^{-4}$. The pesudo code for the solution algorithm of LRTC-TNN is summarized in  Figure~\ref{lrtc_flow}.

Here, the learning rate parameter $\rho$ and the truncation rate parameter $\theta$ determine the performance directly. By using cross validation, we set the parameter $\theta$ for (G), (B), (H), and (S) with RM scenario as 30\%, 15\%, 10\%, and 30\%, respectively. In terms of NM scenario, the parameter $\theta$ for (G), (B), (H), and (S) as 5\%, 5\%, 10\%, and 5\%, respectively. For learning rate parameter $\rho$, it determines the convergence of the whole model. Larger $\rho$ usually slows down the convergence process, while a smaller one would let the model meet convergence in only few iterations. The rate parameter $\theta$ for NM data should be set much smaller than RM data, wherein we set the smallest $\theta$ as 5\% for these data sets. The maximum number of iteration of LRTC-TNN for these imputation experiments is set to 200, and this value is enough to reach convergence in our experiments. The adapted data sets and Python implementation for these experiments are available in the \emph{transdim} GitHub repository \url{https://github.com/xinychen/transdim}.

% {\color{red} How did you choose these hyperparameters?}
%We could see from Table~\ref{hyperparameters} that these hyperparameters are not difficult to set.

% \begin{table}[!ht]
% \caption{Parameters $\rho$ and $\theta$ for imputation tasks on data sets (G), (B), (H), and (S).}
% \label{parameters}
% \centering
% \footnotesize
% % \scriptsize
% \begin{tabular}{lcccccccc}
% \toprule
% & 20\%, RM-G & 40\%, RM-G & 20\%, NM-G & 40\%, NM-G & 10\%, RM-B & 30\%, RM-B & 10\%, NM-B & 30\%, NM-B \\
% \midrule
% % $\rho$ ($\times 10^{-4}$) & 20 & 20 & 20 & 5 & 5 & 5 & 2 & 2 \\
% $\theta$ (\%) & 30 & 30 & 5 & 5 & 15 & 15 & 5 & 5 \\
% \toprule
% & 20\%, RM-H & 40\%, RM-H & 20\%, NM-H & 40\%, NM-H & 20\%, RM-S & 40\%, RM-S & 20\%, NM-S & 40\%, NM-S \\
% \midrule
% % $\rho$ ($\times 10^{-4}$) & 2 & 2 & 2 & 2 & 10 & 10 & 5 & 5 \\
% $\theta$ (\%) & 10 & 10 & 10 & 10 & 30 & 30 & 5 & 5 \\
% \bottomrule
% \end{tabular}
% \end{table}

% Note that even though we only evaluate spatiotemporal data imputation tasks with third-order tensors, the proposed tensor completion can be easily adapted to tensors with any order and the compiled kernel function in our experiment also supports this.

\subsection{Results}

Table~\ref{imputation_result} summarizes the imputation performance of matrix factorization models (BTMF and TRMF), tensor factorization models (BGCP and BATF), and LRTC models (HaLRTC and LRTC-TNN) on the four spatiotemporal data sets. We follow the same missing rates as our previous work \citep{chen2019bayesian} and configure all matrix and tensor factorization models with the same number of factors and BTMF and TRMF with the same time lags. The number of factors of factorization models (i.e., BTMF, TRMF, BGCP, and BATF) are set to 80, 30, 50, 50 for RM data sets (G), (B), (H), and (S), respectively. For the NM scenario, we set the number of factors to 10 for all data sets. The time lag of BTMF and TRMF for all four data sets is $\mathcal{L}=\{1,2,T\}$ where $T$ is the number of time intervals per day.

Overall, LRTC-TNN clearly outperforms other baseline models in diverse missing scenarios (RM and NM scenarios with varying missing rates). For all these models, we see that the MAPE/RMSE values are higher for NM scenario than RM scenario, suggesting that NM scenario is indeed more difficult to handle than the RM scenario. Another fact is that the MAPE/RMSE values become higher with increasing missing rates. Both missing pattern and missing rate have a direct impact on all these models.

For data sets (G), (H) and (S), LRTC-TNN consistently outperforms other baseline models significantly. For the Birmingham (B) parking data set, LRTC-TNN achieves the best imputation performance for the scenario of NM, while the two temporal factorization models BTMF and TRMF performs the best for the RM scenario. This is due to the strong temporal patterns and consistency in data set (B), and thus the temporal smoothness and dynamics (e.g., AR/VAR) play a more important role in capturing the true signal.

Both the two LRTC models (HaLRTC and LRTC-TNN) achieve promising performance; however, we find that LRTC-TNN consistently outperforms HaLRTC with much lower MAPE/RMSE values. On the other hand, the MAPE/RMSE values of using LRTC-TNN do not become higher dramatically as HaLRTC when missing rate increases, and this is also indicated in Table~\ref{high_missing_result}. These results clearly show the advantage of using TNN minimization over NN minimization for learning low-rank tensor structures and retaining dominant data patterns.

% TNN minimization can significantly reinforce the capability of rank approximation using singular values.

% On the Guangzhou traffic speed data set, we see from Table~\ref{imputation_result} that, across various missing patterns and rates, LRTC-TNN achieves the best imputation performance. If we view HaLRTC as a LRTC model with nuclear norm minimization, the comparison between LRTC-TNN and HaLRTC just shows that the TNN minimization can significantly reinforce the capability and information of nuclear norm.

\begin{table}[!ht]
\caption{Performance comparison (in MAPE/RMSE) for RM and NM for imputation tasks on data sets (G), (B), (H), and (S).}
\label{imputation_result}
\centering
\footnotesize
% \scriptsize
\begin{tabular}{lcccccccc}
\toprule
& {BTMF} & TRMF & {BGCP} & {BATF} & HaLRTC & LRTC-TNN \\
\midrule
20\%, RM-G & 7.47/3.19 & 7.47/{3.14} & {8.28/3.57} & {8.32/3.59} & 8.13/3.33 & \textbf{6.70}/\textbf{2.88} \\
40\%, RM-G & 7.81/3.35 & 7.76/3.25 & {8.29/3.59} & {8.36/3.61} & 8.86/3.61 & \textbf{7.32}/\textbf{3.17} \\
20\%, NM-G & {10.16}/4.27 & 10.24/4.27 & {10.20/4.27} & {{10.17}/{4.26}} & 10.45/{4.21} & \textbf{9.37}/\textbf{3.97} \\
40\%, NM-G & 10.36/4.46 & 10.37/4.37 & {10.25/4.32} & {{10.17}/{4.30}} & 10.88/4.38 & \textbf{9.54}/\textbf{4.08} \\
% \midrule
10\%, RM-B & \textbf{1.71}/\textbf{7.44} & 2.77/10.57 & {6.50/19.69} & {6.93/20.65} & 4.78/17.18 & 4.21/13.11 \\
30\%, RM-B & \textbf{2.61}/\textbf{13.38} & 3.69/21.80 & {6.23/19.98} & {6.68/21.29} & 6.59/26.63 & 5.15/17.47 \\
10\%, NM-B & {12.05}/28.27 & 12.74/29.46 & {13.64/43.15} & {16.28/40.81} & \textbf{9.38}/34.52 & 9.40/\textbf{23.26} \\
30\%, NM-B & 15.44/61.69 & 16.35/85.98 & {15.93/{57.07}} & {15.95/{57.07}} & {14.69}/92.66 & \textbf{13.31}/\textbf{52.32} \\
20\%, RM-H & 25.18/{28.51} & 21.31/37.07 & {19.01/41.16} & {22.74/33.07} & {18.27}/28.87 & \textbf{18.03}/\textbf{24.90} \\
40\%, RM-H & 26.83/32.19 & 22.89/38.15 & {19.59/32.71} & {23.17/{31.62}} & {19.02}/31.81 & \textbf{18.80}/\textbf{25.90} \\
20\%, NM-H & 26.50/81.73 & 26.07/40.06 & {25.57/35.99} & {34.94/{29.32}} & {20.30}/40.51 & \textbf{19.71}/\textbf{27.42} \\
40\%, NM-H & 30.24/80.53 & 27.32/{39.75} & {24.37/49.64} & {30.63/48.01} & {{21.46}}/53.15 & \textbf{20.43}/\textbf{29.04} \\
20\%, RM-S & {5.92}/3.71 & 5.96/3.71 & 7.45/4.50 & 8.70/3.73 & 5.93/{3.47} & \textbf{4.65}/\textbf{3.06} \\
40\%, RM-S & 6.18/3.79 & {6.16}/3.79 & 7.58/4.54 & 8.73/{3.75} & 6.76/3.83 & \textbf{5.12}/\textbf{3.30} \\
20\%, NM-S & 9.12/5.27 & 9.12/5.26 & 9.93/5.65 & 10.15/{4.25} & {8.79}/4.69 & \textbf{6.93}/\textbf{4.19} \\
40\%, NM-S & 9.20/5.33 & {9.19}/5.30 & 9.94/5.68 & 10.19/5.27 & 10.19/5.27 & \textbf{7.59}/\textbf{4.50} \\
\bottomrule
\multicolumn{5}{l}{\scriptsize{Best results are highlighted in bold fonts.}}
\end{tabular}
\end{table}

In general, the imputation problem becomes more and more challenging with increasing missing rate. We next investigate the performance of different models under heavy and extreme missing scenarios. In doing so, we compare LRTC-TNN with three baseline models: BTMF as a matrix factorization model, BGCP as a tensor factorization model, and HaLRTC as a low-rank completion model. Table~\ref{high_missing_result} shows the performance of the four models in both RM and NM scenarios under high missing rates (i.e., 50\%, 60\%, and 70\%). As can be seen, LRTC-TNN achieves the lowest MAPE/RMSE values in most cases. It should be noted that HaLRTC becomes unstable in heavy missing scenarios. On the contrary, LRTC-TNN is robust and capable of addressing heavy missing scenarios with high accuracy. These results also imply that the nonconvex TNN minimization shows superior performance over the convex NN minimization.

\begin{table}[!ht]
\caption{Imputation comparison (in MAPE/RMSE) among BTMF, BGCP, HaLRTC, and LRTC-TNN with high missing rates.}
\label{high_missing_result}
\centering
\footnotesize
% \scriptsize
\begin{tabular}{lcccccccc}
\toprule
& BTMF & BGCP & HaLRTC & LRTC-TNN \\
\midrule
50\%, RM-G & \textbf{7.56}/\textbf{3.25} & 9.31/3.77 & 9.30/3.77 & 7.68/3.32 \\
60\%, RM-G & \textbf{7.93}/\textbf{3.41} & 9.83/3.96 & 9.82/3.96 & 8.12/3.51 \\
70\%, RM-G & \textbf{8.38}/\textbf{3.61} & 10.45/4.18 & 10.45/4.18 & 8.61/3.72 \\
50\%, NM-G & 9.87/4.17 & 11.31/4.52 & 11.30/4.52 & \textbf{9.74}/\textbf{4.16} \\
60\%, NM-G & 10.05/\textbf{4.24} & 11.81/4.69 & 11.80/4.69 & \textbf{9.92}/\textbf{4.24} \\
70\%, NM-G & 10.87/4.52 & 12.65/4.96 & 12.66/4.97 & \textbf{10.21}/\textbf{4.33} \\
50\%, RM-B & \textbf{6.03}/\textbf{20.97} & 9.16/39.96 & 9.12/39.85 & 6.52/22.68 \\
60\%, RM-B & \textbf{6.66}/\textbf{24.53} & 11.17/55.14 & 11.14/55.02 & 7.62/27.15 \\
70\%, RM-B & 9.22/\textbf{32.31} & 14.13/73.58 & 14.13/73.43 & \textbf{8.95}/33.45 \\
50\%, NM-B & 15.35/\textbf{85.93} & 19.28/195.53 & 19.24/194.82 & \textbf{15.32}/86.95 \\
60\%, NM-B & \textbf{17.91}/100.81 & 24.13/301.79 & 24.12/303.42 & 18.01/\textbf{96.18} \\
70\%, NM-B & \textbf{20.39}/\textbf{115.67} & 29.32/385.27 & 29.34/387.76 & 20.95/145.29 \\
50\%, RM-H & 19.45/\textbf{26.79} & 19.51/33.26 & 19.51/33.26 & \textbf{19.26}/26.86 \\
60\%, RM-H & 19.85/28.33 & 20.13/36.19 & 20.09/36.19 & \textbf{19.56}/\textbf{27.84} \\
70\%, RM-H & 20.57/30.34 & 21.01/40.08 & 20.95/40.08 & \textbf{20.34}/\textbf{29.90} \\
50\%, NM-H & 22.27/34.42 & 22.89/60.86 & 22.88/60.63 & \textbf{21.22}/\textbf{30.68} \\
60\%, NM-H & 21.86/40.30 & 23.95/92.28 & 23.93/91.92 & \textbf{21.22}/\textbf{37.67} \\
70\%, NM-H & 22.17/42.83 & 26.30/108.37 & 26.23/107.63 & \textbf{21.29}/\textbf{39.70} \\
50\%, RM-S & \textbf{5.27}/\textbf{3.36} & 7.31/4.08 & 7.30/4.07 & 5.43/3.47 \\
60\%, RM-S & \textbf{5.59}/\textbf{3.52} & 7.91/4.34 & 7.90/4.34 & 5.80/3.66 \\
70\%, RM-S & \textbf{6.19}/\textbf{3.84} & 8.89/4.76 & 8.89/4.76 & 6.53/4.04 \\
50\%, NM-S & 8.15/4.74 & 11.20/5.64 & 11.19/5.64 & \textbf{8.12}/\textbf{4.73} \\
60\%, NM-S & 8.69/5.03 & 12.38/6.11 & 12.37/6.11 & \textbf{8.61}/\textbf{5.00} \\
70\%, NM-S & 9.93/5.62 & 14.30/6.91 & 14.34/6.94 & \textbf{9.41}/\textbf{5.45} \\
\bottomrule
\multicolumn{5}{l}{\scriptsize{Best results are highlighted in bold fonts.}}
\end{tabular}
\end{table}

% \begin{figure}
% \centering
%   \includegraphics[width=0.9\textwidth]{high_missing_result.pdf}
% \caption{Imputation comparison among BTMF, BGCP, HaLRTC, and LRTC-TNN with high missing rates (i.e., 50\%, 60\%, and 70\%).}
% \label{high_missing_result}
% \end{figure}

% {\color{red} here}
We next choose some examples from the four experiment data sets and visualize the NM time series with the extreme missing rate $70\%$ and the corresponding recovered time series by using LRTC-TNN in Figure~\ref{curve}. We see that LRTC-TNN can achieve very high accuracy for all the four data sets with only 30\% input.

% Figure~\ref{curve} shows the imputation result of xx.

\begin{figure}
\centering
  \includegraphics[width=0.95\textwidth]{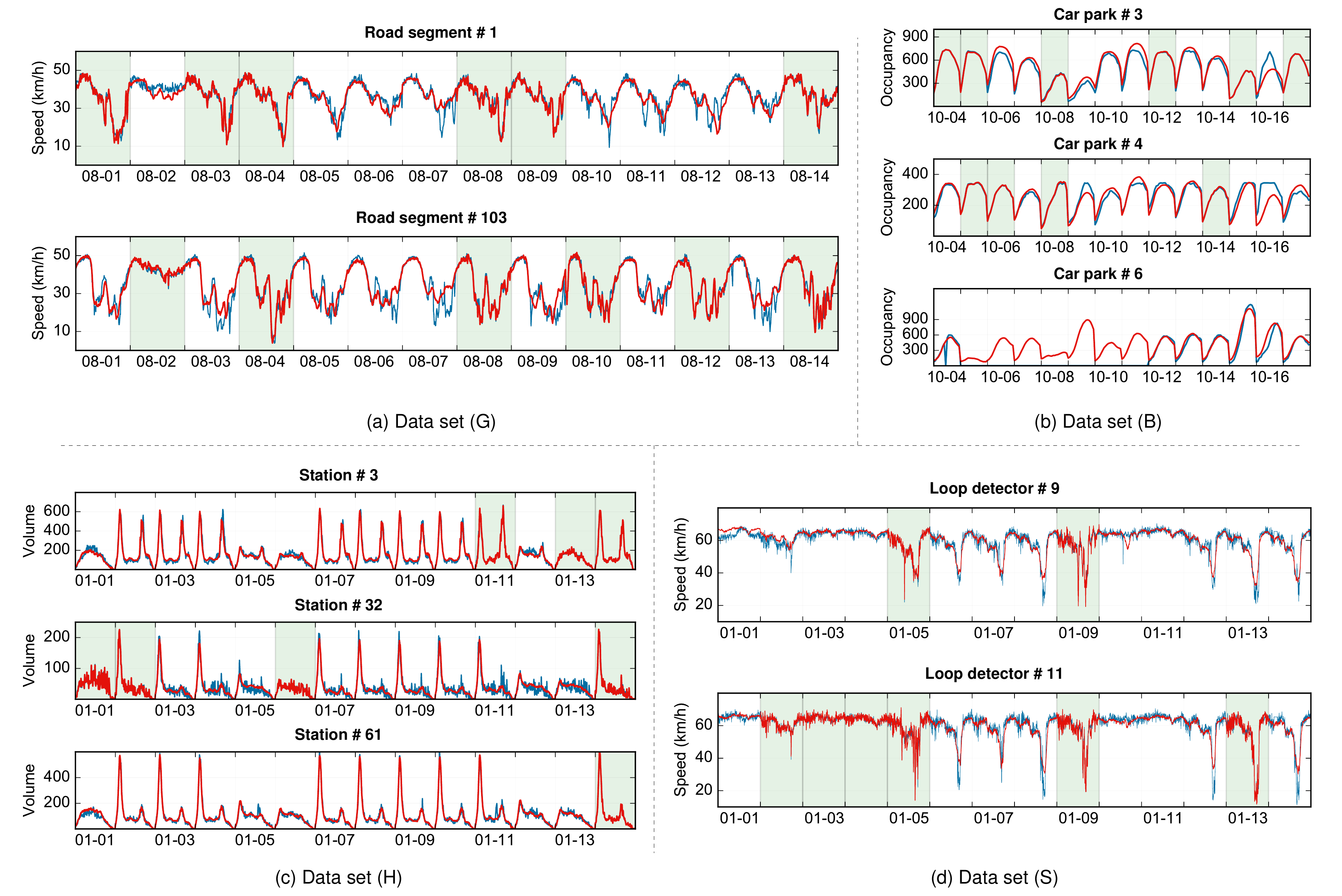}
\caption{Imputation examples for the four traffic data sets. In these panels, red curves are the imputed values, while blue curves are the ground truth. White rectangles represent fiber missing (i.e., observations are missing in a whole day), and green rectangles show the partially observed data.}
\label{curve}
\end{figure}

\section{Conclusion and Future Directions} \label{sec:conclusion}

In this paper, we propose a nonconvex TNN minimization based low-rank tensor completion model---LRTC-TNN---to impute missing values in spatiotemporal traffic data. Our experiments show that our proposed model consistently outperforms other state-of-the-art imputation models on the four selected real-world traffic data sets with two representative missing patterns (i.e., random missing [RM] and non-random missing [NM]). The superiority of tensor representation (with additional ``day'' dimension) over matrix-based models further verifies the strong and unique low-rank property of traffic time series data at different temporal scales \citep{li2015trend}.

% {\color{red} (not very clear) LRTC-TNN has two major advantages over previous methods: (1) it is flexible to truncate tensors' nuclear norm by a truncation rate parameter; on the other hand, the truncation is able to keep the major components fixed at each iteration.}

For future research, we propose the following directions:
\begin{itemize}

    \item\emph{Incorporating side information and smoothness priors}: Spatiotemporal data demonstrates both global consistency and local consistency (e.g., the autoregressive patterns in BTMF and TRMF). The underlying low-rank tensor is effective to capture the global consistency. To further impose local consistency and better characterize side information, we can introduce smoothness priors to the model by introducing additional cost terms or constraints into the optimization framework, such as spatial Laplacian regularization, temporal Toeplitz regularization, and total variation \citep{roughan2011spatio,yokota2016smooth}.

    % {\color{red} Add the formulation of objective functions.}

    For instance, we can encode a priori information on the spatial structure using a graph Laplacian matrix $\boldsymbol{\Lambda}$ where $\boldsymbol{\Lambda}=\boldsymbol{D}-\boldsymbol{A}$ and $\boldsymbol{A},\boldsymbol{D}$ are the adjacent matrix of spatial structure and the degree matrix, respectively. This allows us to impose the similarity between two connected locations by introducing the penalty  $\frac{\beta}{2}\operatorname{tr}(\boldsymbol{X}^\top\boldsymbol{\Lambda}\boldsymbol{X})$ where rows of any matrix $\boldsymbol{X}$ correspond to spatial locations. Combining with autoregression or smoothness on temporal dimension, we can formulate a spatiotemporal regularization on the variable $\boldsymbol{\mathcal{M}}^{l+1}$ (see the subproblem in Eq.~\eqref{optimal_M_variable}).

% \begin{equation}
%     \begin{aligned}
%     \boldsymbol{\mathcal{M}}^{l+1}:&=\operatorname{arg}\min_{\boldsymbol{\mathcal{M}}}~\frac{\beta_k}{2}\operatorname{tr}(\boldsymbol{\mathcal{Z}}_{(k)}^\top\boldsymbol{\Lambda}\boldsymbol{\mathcal{Z}}_{(k)})+\frac{\rho_k}{2}\|\boldsymbol{\mathcal{X}}_{(k)}^{l+1}-\boldsymbol{\mathcal{Z}}_{(k)}\|_{F}^{2}-\big\langle\boldsymbol{\mathcal{Z}}_{(k)},\boldsymbol{\mathcal{T}}_{k(k)}^{l}\big\rangle \\
%     &=\operatorname{arg}\min_{\boldsymbol{\mathcal{Z}}}~\frac{\beta_k}{2}\operatorname{tr}(\boldsymbol{\mathcal{Z}}_{(k)}^\top\boldsymbol{\Lambda}\boldsymbol{\mathcal{Z}}_{(k)})+\frac{\rho_k}{2}\|\boldsymbol{\mathcal{Z}}_{(k)}\|_{F}^{2}-\rho_k\big\langle\boldsymbol{\mathcal{Z}}_{(k)},\boldsymbol{\mathcal{X}}_{k(k)}^{l+1}+\frac{1}{\rho_k}\boldsymbol{\mathcal{T}}_{k(k)}^{l}\big\rangle. \\
%     &=\operatorname{arg}\min_{\boldsymbol{\mathcal{Z}}}~\frac{\beta_k}{2}\operatorname{tr}(\boldsymbol{\mathcal{Z}}_{(k)}^\top\boldsymbol{\Lambda}\boldsymbol{\mathcal{Z}}_{(k)})+\frac{\rho_k}{2}\|\boldsymbol{\mathcal{Z}}_{(k)}-(\boldsymbol{\mathcal{X}}_{k(k)}^{l+1}+\frac{1}{\rho_k}\boldsymbol{\mathcal{T}}_{k(k)}^{l})\|_{F}^{2}, \\
%     \Rightarrow \boldsymbol{\mathcal{Z}}_{k(k)}^{l+1}&=(\rho_k\boldsymbol{I}+\beta_k\boldsymbol{\Lambda})^{-1}(\rho_k\boldsymbol{\mathcal{X}}_{k}^{l+1}+\boldsymbol{\mathcal{T}}_{k}^{l}), \\
%     \end{aligned}
%     \label{optimal_z_laplacian}
% \end{equation}
% where $\boldsymbol{I}$ is an identity matrix.

    \item\emph{Minimizing Schatten $p$-norm}: In algebra, Schatten $p$-norm is defined as
    \begin{equation}
        \|\boldsymbol{X}\|_{S_p}=\left[\sum_{i}\sigma_{i}^{p}(\boldsymbol{X})\right]^{1/p},
    \end{equation}
    where $\boldsymbol{X}\in\mathbb{R}^{m\times n}$ and $p>0$ \citep{fan2019factor}. Here, we could see that: (1) If $p\to 0$, then Schatten $p$-norm is the rank function, i.e., $\text{rank}(\boldsymbol{X})$; (2) If $p=1$, then Schatten $p$-norm is the nuclear norm. Therefore, we can find a better surrogate for the rank function by using the Schatten $p$-norm than the NN by letting $0<p<1$. Accordingly, truncated Schatten $p$-norm that stems from Schatten $p$-norm might be a potentially better option than TNN. However, a critical bottleneck for Schatten $p$-norm minimization is the nonconvex optimization, which is more difficult to solve than TNN minimization discussed in this work.

    \item\emph{Minimizing weighted nuclear norm}: Instead of using TNN minimization, we can develop a weighted version of the NN minimization to better account for the differences/variations between singular values \citep{gu2014weighted}. The objective function of our problem can also be defined by a weighted version of the NN minimization:
\begin{equation}
    \|\boldsymbol{{X}}\|_{\boldsymbol{w},*}=\sum_{i}w_{i}\sigma_{i}(\boldsymbol{{X}}),
    \label{weighted_NN}
\end{equation}
where $\boldsymbol{w}=[{w}_{1},{w}_{2},\cdots,w_{\min\{m,n\}}]^\top$ are non-negative weights assigned to the singular values $\sigma_1,\sigma_2,...,\sigma_{\min\{m,n\}}$. This will lead a more general nonconvex form; however, this will introduce an additional set of weight parameters and it becomes a  challenge to tune these parameters for a specific data set.

    \item\emph{Developing new formulations for tensor rank approximation}: The Schatten $p$-norm and weighted nuclear norm are two alternatives to truncated nuclear norm for matrices. However, it remains a critical challenge to define a proper norm for traffic data tensors. This study employs the tensor nuclear norm defined in Eq.~\eqref{LRTC_NN} and we simply use the same $\alpha_k$ for all the three unfoldings. However, different from image data, the dimensions for traffic tensor (i.e., sensor, time of day, and day) are heterogeneous and cannot be compared directly. As a result, we may expect the three unfoldings contribute differently to the overall rank approximation. A potential future research direction is to explore new alternatives for tensor rank approximation.
\end{itemize}

\section*{Acknowledgement}
This research is supported by the Natural Sciences and Engineering Research Council (NSERC) of Canada, the Fonds de recherche du Quebec – Nature et technologies (FRQNT), and the Canada Foundation for Innovation (CFI). X. Chen would  like  to  thank  the  Institute  for  Data  Valorisation(IVADO) for providing scholarship to support this study.

\end{document}